%% file: paper.tex
\documentclass[runningheads]{llncs}

\usepackage[T1]{fontenc}
\usepackage{cite}
\usepackage{amsmath,amssymb,amsfonts}
\usepackage{graphicx}
\usepackage{textcomp}

\usepackage{multirow}
\usepackage[dvipsnames]{xcolor}
\usepackage{etoolbox}
\usepackage{amsmath}
\usepackage{latexsym,amssymb}
\usepackage{upgreek}
\usepackage[noend]{algpseudocode}
\usepackage{array}
\usepackage{adjustbox}
\usepackage{cite}
\usepackage{soul}
\usepackage{url}
\usepackage[utf8]{inputenc}
\usepackage[small]{caption}
\usepackage{booktabs}
\usepackage{empheq}
\usepackage{fancybox}
\usepackage{subcaption}
\usepackage[english]{babel}
\usepackage{xspace}
\usepackage{setspace}
\usepackage{stmaryrd}
\usepackage{enumitem}
\usepackage{ifthen}
\usepackage{verbatim}
\usepackage{nicefrac}       
\usepackage{algorithm}
\usepackage{algorithmicx}
\usepackage{relsize}

\usepackage[mode=text]{siunitx}
\usepackage{bm}

\let\llncssubparagraph\subparagraph
\let\subparagraph\paragraph
\usepackage{titlesec}
\let\subparagraph\llncssubparagraph

\usepackage{amsthm}
\usepackage{hyperref}
\usepackage[capitalise,nameinlink]{cleveref}

\usepackage{tikz}
\usepackage{forest}

\usetikzlibrary{arrows,decorations.pathmorphing,decorations.footprints,fadings,calc,trees,mindmap,shadows,decorations.text,patterns,positioning,shapes,matrix,fit}
\usetikzlibrary{arrows,shadows,backgrounds}
\usetikzlibrary{arrows.meta}
\usetikzlibrary{positioning,fit}
\usetikzlibrary{automata}
\usetikzlibrary{matrix}
\usetikzlibrary{shapes.symbols,shapes.misc,shapes.arrows,shapes.geometric}
\usetikzlibrary{matrix,chains,scopes,decorations.pathmorphing}

\input{defs}

\begin{document}

\input{preamble}

\maketitle

\input{abs}

\input{intro}

\input{prelim}
\input{relw}

\input{cons}

\input{pros}

\input{res}

\input{conc}

\ifthenelse{\boolean{nonanon}}
{\input{acks}}{}

\input{replbib}
\input{togbbl} 

\iftoggle{mkbbl}{
    \bibliographystyle{splncs04}
    \bibliography{refs,team,refs2,xtra,xtra2}
}{

\input{paper.bibl}
}


\end{document}

%% file: defs.tex
\newboolean{nonanon}              
\setboolean{nonanon}{true}        %
\newboolean{mkcompact}            
\setboolean{mkcompact}{false}     %
\newboolean{squeezeit}            
\setboolean{squeezeit}{true}      %
\newboolean{ispreprint}           
\setboolean{ispreprint}{false}    %
\hypersetup{
  colorlinks = true,
  linkcolor = MidnightBlue,
  urlcolor  = MidnightBlue,
  citecolor = MidnightBlue,
  anchorcolor = MidnightBlue
}

\crefname{prop}{Proposition}{Propositions}
\Crefname{prop}{Proposition}{Propositions}
\crefname{defn}{Definition}{Definitions}
\Crefname{defn}{Definition}{Definitions}
\crefname{cor}{Corollary}{Corollaries}
\Crefname{cor}{Corollary}{Corollaries}
\crefname{exmpl}{Example}{Examples}
\Crefname{exmpl}{Example}{Examples}
\Crefname{theorem}{Theorem}{Theorem}

\newtheoremstyle{nutstyle}
{3.0pt} 
{3.0pt} 
{} 
{} 
{\bfseries} 
{.} 
{.5em} 
{} 

\theoremstyle{nutstyle}

\newtheoremstyle{nuxstyle}
{3.0pt} 
{3.0pt} 
{} 
{} 
{\itshape} 
{.} 
{.5em} 
{} 

\theoremstyle{nuxstyle}

\newcommand{\todoF}[2]{}

\newcommand{\fml}[1]{{\mathcal{#1}}}

\newcommand{\tn}[1]{\textnormal{#1}}
\newcommand{\mbf}[1]{\ensuremath\mathbf{#1}}
\newcommand{\msf}[1]{\ensuremath\mathsf{#1}}
\newcommand{\mbb}[1]{\ensuremath\mathbb{#1}}

\newcommand{\tbf}[1]{\textbf{#1}}

\newcommand{\nfrac}{\nicefrac}
\newcommand{\eaxp}{$\varepsilon\tn{AXp}$\xspace}
\newcommand{\ecxp}{$\varepsilon\tn{CXp}$\xspace}
\newcommand{\eaex}{$\varepsilon\tn{AEx}$\xspace}
\newcommand{\eaxps}{$\varepsilon\tn{AXps}$\xspace}
\newcommand{\ecxps}{$\varepsilon\tn{CXps}$\xspace}

\newcommand{\beaxps}{$\varepsilon\tbf{AXps}$\xspace}
\newcommand{\becxps}{$\varepsilon\tbf{CXps}$\xspace}

\newcommand{\TRUE}{\textbf{true}\xspace}
\newcommand{\FALSE}{\textbf{false}\xspace}
\newcommand{\SAT}{\ensuremath\mathsf{SAT}}
\newcommand{\outc}{\ensuremath\mathsf{outc}}

\newcommand{\prtaxp}{\ensuremath\mathsf{reportAXp}}
\newcommand{\prtcxp}{\ensuremath\mathsf{reportCXp}}

\newcommand{\findaxpdel}{\ensuremath\msf{FindAXpDel}}

\newcommand{\findcxpdel}{\ensuremath\msf{FindCXpDel}}
\newcommand{\robt}{\ensuremath\msf{FindAEx}}

\definecolor{darkslategray}{rgb}{0.18, 0.31, 0.31} 
\definecolor{platinum}{rgb}{0.9, 0.89, 0.89} 
\definecolor{gray}{rgb}{.4,.4,.4}
\definecolor{midgrey}{rgb}{0.5,0.5,0.5}
\definecolor{middarkgrey}{rgb}{0.35,0.35,0.35}
\definecolor{darkgrey}{rgb}{0.3,0.3,0.3}
\definecolor{darkred}{rgb}{0.7,0.1,0.1}
\definecolor{midblue}{rgb}{0.2,0.2,0.7}
\definecolor{darkblue}{rgb}{0.1,0.1,0.5}
\definecolor{defseagreen}{cmyk}{0.69,0,0.50,0}
\newcommand{\jnote}[1]{\medskip\noindent$\llbracket$\textcolor{darkred}{joao}: \emph{\textcolor{middarkgrey}{#1}}$\rrbracket$\medskip}
\newcommand{\ynote}[1]{\medskip\noindent$\llbracket$\textcolor{darkblue}{yacine}: \emph{\textcolor{middarkgrey}{#1}}$\rrbracket$\medskip}

\newcommand{\jnoteF}[1]{}
\newcommand{\anoteF}[1]{}

\newcounter{tableeqn}[table]

\DeclareMathOperator*{\limply}{\rightarrow}

\newcommand{\waxp}{\ensuremath\mathsf{WAXp}}
\newcommand{\wcxp}{\ensuremath\mathsf{WCXp}}
\newcommand{\axp}{\ensuremath\mathsf{AXp}}
\newcommand{\cxp}{\ensuremath\mathsf{CXp}}

\newcommand{\waxpe}{\ensuremath\varepsilon\mathsf{WAXp}}

\newcommand{\mailtodomain}[1]{\href{mailto:#1@ciencias.ulisboa.pt}{\texttt{\nolinkurl{#1}}}}

\titleformat{\paragraph}[runin]
{\normalfont\bfseries}{}{0pt}{}

\titlespacing{\paragraph}{0pt}{*3.25}{*1.05}

\newcolumntype{L}[1]{>{\raggedright\let\newline\\\arraybackslash\hspace{0pt}}m{#1}}
\newcolumntype{C}[1]{>{\centering\let\newline\\\arraybackslash\hspace{0pt}}m{#1}}
\newcolumntype{R}[1]{>{\raggedleft\let\newline\\\arraybackslash\hspace{0pt}}m{#1}}

\newenvironment{Proof}{\noindent{\em Proof.~}}{\hfill$\Box$\\[-0.15cm]}

\ifthenelse{\boolean{nonanon}}{
  
}{
  
}

%% file: preamble.tex
\title{The Pros and Cons of Adversarial Robustness} 
\ifthenelse{\boolean{nonanon}}{
  \titlerunning{The Pros \& Cons of Robustness}
}{
  \titlerunning{}
}
%

\ifthenelse{\boolean{nonanon}}{
  \author{%
    Yacine Izza\inst{1}\orcidID{0000-0002-7774-1945}
    \and\\
    Joao Marques-Silva\inst{2}\orcidID{0000-0002-6632-3086}
  }
  \authorrunning{Y. Izza \& J. Marques-Silva}
  %
  \institute{%
    CREATE, NUS, Singapore
    \email{izza@comp.nus.edu.sg}
    \and
    IRIT, CNRS, Toulouse, France
    \email{joao.marques-silva@irit.fr}
  }
}{
  \author{\tbf{Paper \# NNNN}}
  \authorrunning{~~}
  \institute{}
}

%% file: abs.tex
\begin{abstract}
  Robustness is widely regarded as a fundamental problem in the
  analysis of machine learning (ML) models. Most often robustness
  equates with deciding the non-existence of adversarial examples,
  where adversarial examples denote situations where small changes on
  some inputs cause a change in the prediction. 
  The perceived importance of ML model robustness explains the
  continued progress observed
  for most of the last decade.
  Whereas robustness is often assessed locally, i.e. given some target
  point in feature space, robustness can also be defined globally,
  i.e. where any point in feature space can be considered.
  The importance of ML model robustness is illustrated for example by
  the existence of competitions evaluating the progress of robustness
  tools, namely in the case of neural networks (NNs) but also by
  efforts towards robustness certification.
  More recently, robustness tools have also been used for computing
  rigorous explanations of ML models.
  In contrast with the observed successes of robustness, this paper
  uncovers some limitations with existing definitions of robustness,
  both global and local, but also with efforts towards robustness
  certification.
  The paper also investigates uses of adversarial examples besides
  those related with robustness.
\end{abstract}


\keywords{%
  Local \& Global Robustness; Certified Robustness;
  Adversarial Examples; Explanations }

%% file: intro.tex
\section{Introduction} \label{sec:intro}

For more than a decade, Machine Learning (ML) has been the subject of
remarkable advances. However, such advances have also been marred by a
number of persistent challenges.
Arguably the best known of these challenges is the \emph{brittleness}
of ML models~\cite{hein-nips17}.
An ML model is brittle if it exhibits \emph{adversarial examples}
(AExs), i.e.\ small changes to the inputs of the ML model can cause
(unexpected and unwanted) changes to the
prediction~\cite{biggio-ecml13,goodfellow-iclr14,goodfellow-iclr15}. Intuitively,
an ML model is robust if it exhibits no AExs.
ML model robustness has been extensively studied over the last
decade~\cite{goodfellow-iclr15,mao-acmcs22}.
The importance of deciding the robustness of ML models motivated an
outpouring of competing approaches, ranging for rather informal 
solutions, to those based on automated reasoners, and even those based
on domain-specific reasoners.
Furthermore, the significance of assessing and asserting robustness is
further highlighted by VNN-COMP (Verification of Neural Networks
Competition~\cite{johnson-sttt23}), a competition of robustness tools
for NNs, held since 2020. In addition, there are also efforts
targeting the \emph{robustness certification} of ML
models~\cite{kolter-icml19a,kolter-icml20a,kolter-iclr20a,kolter-nips21,kolter-iclr23a}. 
Additional efforts towards the verification and validation of complex
ML models include different testing
approaches~\cite{jana-sosp17,jana-icse18,kroening-ase18,jana-cacm19,kroening-icse19,kroening-csr20}.
Furthermore, there have been proposals towards the verification and
validation of systems based on artificial
intelligence~\cite{seshia-cacm22}, covering not only robustness, but
also explainability and fairness.
The uses of systems of artificial intelligence in high-risk and
safety-critical domains had motivated calls for the use of so-called
\emph{interpretable models}~\cite{rudin-naturemi19,rudin-ss22}, with
the purpose of enabling human-decision makers to explain the decisions
taken by such systems.
Unfortunately, such calls have not deterred proposal for the use of
complex systems of AI in high-risk and safety-critical domains.

Robustness is often defined with respect to a concrete input to the ML
model and its associated prediction. In this case, one is referring to
what is called \emph{local robustness}. An alternative view is
\emph{global robustness}, where the goal is to assess robustness for
\emph{any} input of the ML model. Nevertheless, there exist tools that
target both local and global robustness~\cite{barrett-cav17}.
Furthermore, to understand whether an ML is locally/globally robust,
it is also fundamental to outline an adequate experimental setup.

More recently, robustness has been linked with formal approaches to
explainable artificial intelligence
(XAI)~\cite{barrett-corr22,hms-corr23b}. For example, it is the case
that so-called abductive explanations~\cite{msi-aaai22} are such that
no adversarial examples can be identified. More importantly, the
relationship between explanations and adversarial examples allowed for
efficient explainability algorithms to be devised for complex ML
classifiers, including neural
networks~\cite{barrett-corr22,hms-corr23b}. Given the
importance of XAI, such proposals further increase the importance of
robustness tools.

This paper argues that existing definitions of local and global
robustness are problematic. More importantly, the paper shows that the
experimental setup used for assessing robustness is invariably
misguided. In a similar fashion, the paper argues that efforts to
deliver \emph{robustness certification} are necessarily futile.
Motivated by these results, the paper hypothesizes that there is no
simple solution to the basic shortcomings of existing approaches for 
deciding the robustness of ML models.
Nevertheless, the paper also underlines the importance of robustness
tools, not specifically for deciding robustness, but instead as
a key building block for the computation of formal explanations by
iteratively deciding the existence of adversarial examples.

The paper is organized as follows.
\cref{sec:prelim} introduces the notation and definitions used
throughout the paper.
Past work on the robustness of ML models is overviewed
in~\cref{sec:relw}.
\cref{sec:cons} presents a number of shortcomings with existing
approaches towards robustness of ML models.
In contrast, \cref{sec:pros} summarizes positive aspects on past work
on robustness.
Afterwards, \cref{sec:res} analyzes illustrative experiments, most of
which related with~\cref{sec:cons}.
Finally,~\cref{sec:conc} concludes the paper.

%% file: prelim.tex
\section{Preliminaries} \label{sec:prelim}

\jnoteF{This stuff is temporary. Parts of it may or may not be
  used. This is not important for now.}

\paragraph{Measures of distance.}
We consider the well-known $l_p$ measure of distance, 
\begin{equation} \label{def:lp}
  ||\mbf{x}-\mbf{y}||_p =
  \left(\sum\nolimits_{i=1}^{m}|x_i-y_i|^{p}\right)^{\nfrac{1}{p}}
\end{equation}
Which is referred as the Minkowski distance.
Special cases include the Manhattan distance $l_1$, the Euclidean
distance $l_2$ and the Chebyshev distance $l_{\infty}$, which
is defined by,
\begin{equation}
  \lim_{p\to\infty}||\mbf{x}-\mbf{y}||_p  = \max_{1\le{i}\le{m}}\{|x_i-y_i|\}
\end{equation}
Moreover, $l_0$ denotes the Hamming distance, defined by:
\begin{equation} \label{def:l0}
  ||\mbf{x}-\mbf{y}||_0 = \sum\nolimits_{i=1}^{m}\tn{ITE}(x_i=y_i,0,1)
\end{equation}

\paragraph{Classification problems.}

A classification problem is defined on a set of features
$\fml{F}=\{1,\ldots,m\}$ and a set of classes
$\fml{K}=\{c_1,c_2,\ldots,c_K\}$.
Each feature $i\in\fml{F}$ takes values from a domain $\fml{D}_i$.
Domains can be categorical or ordinal. If ordinal, domains can be
discrete or real-valued.
Throughout the paper, and unless otherwise stated, domains will be
assumed to be real-valued.
%
Feature space is defined by
$\mbb{F}=\fml{D}_1\times{\fml{D}_2}\times\ldots\times{\fml{D}_m}$. 
The notation $\mbf{x}=(x_1,\ldots,x_m)$ denotes an arbitrary point in 
feature space, where each $x_i$ is a variable taking values from
$\fml{D}_i$. Moreover, the notation $\mbf{v}=(v_1,\ldots,v_m)$
represents a specific point in feature space, where each $v_i$ is a
constant representing one concrete value from $\fml{D}_i$.
An \emph{instance} denotes a pair $(\mbf{v}, c)$, where
$\mbf{v}\in\mbb{F}$ and $c\in\fml{K}$.
An ML classifier $\fml{M}$ is characterized by a
non-constant \emph{classification function} $\kappa$ that maps feature
space $\mbb{F}$ into the set of classes $\fml{K}$,
i.e.\ $\kappa:\mbb{F}\to\fml{K}$.
Given the above, we associate with a classifier $\fml{M}$, a tuple 
$(\fml{F},\mbb{F},\fml{K},\kappa)$.
Since we assume that $\kappa$ is non-constant, then the ML classifier
$\fml{M}$ is declared \emph{non-trivial},
i.e.\ $\exists(\mbf{a},\mbf{b}\in\mbb{F}).\left(\kappa(\mbf{a})\neq\kappa(\mbf{b})\right)$. 


When reasoning about systems with formal methods, it is often assumed
that all inputs are possible, or alternatively, that there exist
explicit constraints that disallow some inputs. In contrast,
reasoning methods rooted in machine learning usually assume some input
distribution, which most often needs to be inferred (or approximated)
from training data or be user-specified.
We say that a point $\mbf{x}$ is \emph{viable} if, for reasoning
purposes, $\mbf{x}$ must be accounted for.
As a result, when making statements about ML models, we consider three
possible scenarios on the inputs:
\begin{enumerate}
\item \emph{Unconstrained inputs}, i.e.\ any point $\mbf{x}$ in
  feature space is viable.
\item \emph{Constrained inputs}, i.e.\ any point $\mbf{x}$ in feature
  space is viable iff some set of constraints $\fml{C}$ is satisfied
  by $\mbf{x}$. This is represented by the predicate $\xi(\mbf{x})$.
\item \emph{Distribution-restricted} inputs, i.e.\ any point $\mbf{x}$
  in feature space is viable iff it respects some distribution
  $\fml{D}$ i.e. $\mbf{x}\sim\fml{D}$, which is either user-specified
  or it is inferred from training data.
\end{enumerate}
Throughout the paper, we assume the case of \emph{unconstrained
inputs} for two main reasons. First, rigorous approaches for
robustness make that implicit assumption. Second, assuming
unconstrained inputs simplifies the notation.
Nevertheless, in cases where the distinction matters, the paper also
accounts for the other possible cases on the inputs.

\subsection{Robustness \& Adversarial Examples} \label{ssec:robae}

\paragraph{Local robustness.}
%
Given a classifier $\fml{M}$ with a classification function $\kappa$,
and an instance $(\mbf{v},c)$, the classifier is \emph{locally
robust} (or just robust) for $\mbf{v}$ if,
\begin{equation} \label{eq:locrob}
  \forall(\mbf{x}\in\mbb{F}).\left[||\mbf{x}-\mbf{v}||_p\le\epsilon\right]\limply\left(\kappa(\mbf{x})=\kappa(\mbf{v})\right)
\end{equation}
If the classifier is not robust, then any point $\mbf{x}\in\mbb{F}$
satisfying the condition,
\begin{equation} \label{eq:aex}
  \left[||\mbf{x}-\mbf{v}||_p\le\epsilon\land\left(\kappa(\mbf{x})\not=\kappa(\mbf{v})\right)\right]
\end{equation}
is referred to as an \emph{adversarial example} for distance
$\epsilon$ (\eaex). (Observe that~\eqref{eq:aex} consists of selecting
one of the counterexamples to~\eqref{eq:locrob}.)

The definitions of local robustness in~\eqref{eq:locrob} and of
adversarial example in~\eqref{eq:aex} assume unconstrained inputs. For
completeness, we include the definitions of local robustness for the
other cases regarding assumptions on the inputs.

For constrained inputs we have,
\begin{equation} \label{eq:locrob2}
  \forall(\mbf{x}\in\mbb{F}).\left[\xi(\mbf{x})\land\left(||\mbf{x}-\mbf{v}||_p\le\epsilon\right)\right]\limply\left(\kappa(\mbf{x})=\kappa(\mbf{v})\right)
\end{equation}
For distribution-restricted inputs we have,
\begin{equation} \label{eq:locrob3}
  \reflectbox{$S$}_{\varsigma}(\mbf{x}\sim\fml{D}).\left[||\mbf{x}-\mbf{v}||_p\le\epsilon\right]\limply\left(\kappa(\mbf{x})=\kappa(\mbf{v})\right)
\end{equation}
where $\reflectbox{$S$}_{\varsigma}$ captures the sampling according
to some distribution and target confidence ($\varsigma$), and where 
$\sim$ is interpreted as a predicate, with two arguments $\mbf{x}$ and
$\fml{D}$, that holds iff $\mbf{x}$ respects the distribution
$\fml{D}$. 

There exist a multitude of proposed robustness tools dedicated to
local robustness, many of which are regularly assessed in the
VNN-COMP~\cite{johnson-sttt23}.
\cref{sec:relw} briefly overviews existing work on robustness.

One additional observation is that tools that exploit automated
reasoners assume unconstrained inputs, and so the definition of local
robustness considered is~\eqref{eq:locrob}.
If information about the context in which the ML models is to be
deployed, then $\xi(\mbf{x})$ may be available, and so it is to be
expected that~\eqref{eq:locrob2} would be used instead.
Finally, incomplete methods, often based on the sampling of feature
space, assume the somewhat different definition~\eqref{eq:locrob3},
which can offer probabilistic guarantees, but not absolute
guarantees.

\jnoteF{ToDo}

\paragraph{Certified robustness.}
Earlier research have also proposed \emph{certified robustness}, which
has been defined as follows: 
\begin{definition}[From~\cite{kolter-icml19a}] \label{def:certrob}
  \emph{``A classifier is said to be certifiably robust if for any
  input  $\mbf{x}$, one can easily obtain a guarantee that the
  classifier's prediction is constant within some set around
  $\mbf{x}$, often an $l_2$ or $l_{\infty}$ ball''.}
\end{definition}
%
%
(We underscore that~\cref{def:certrob} is taken verbatim
from~\cite{kolter-icml19a}.) We will refer to this definition
throughout the paper.

\paragraph{Global robustness.}
Given the definition of local robustness, a possible definition of
\emph{global robustness} is,
\begin{equation} \label{eq:globrob}
  \forall(\mbf{v},\mbf{x}\in\mbb{F}).\left(||\mbf{x}-\mbf{v}||_p\le\epsilon\right)\limply\left(\kappa(\mbf{x})=\kappa(\mbf{v})\right) 
\end{equation}
(Observe that~\eqref{eq:globrob} is just a formalization
of~\cref{def:certrob}, by allowing the norm $l_p$ to be any.)
Similar definitions have been studied in the literature. For example,
Reluplex~\cite{barrett-cav17} defines global robustness by allowing
small differences in predictions. This alternative definition raises
concerns in classifications problems, because predicted values may be
small, but the predicted classes may be different.
Moreover, there are other variants of this definition, which the paper
also studies. However, by default we will consider this apparently
sensible definition throughout the paper.


As shown in~\cref{sec:relw}, other definitions of global robustness
can be related with the one proposed above.
\cref{sec:relw} also briefly overviews approaches for deciding global
robustness, local robustness and certified robustness.

\subsection{Logic-Based Explainability} \label{ssec:xps}

\jnoteF{This is work in progress!}

%
As mentioned in~\cref{sec:intro}, the concepts of adversarial examples
and explainability are tightly related. As a result, we include a
brief introduction to logic-based explainability. More detailed
accounts are available~\cite{msi-aaai22,darwiche-lics23}. 

An explanation problem $\fml{E}$ is a tuple $(\fml{M},(\mbf{v},c))$,
where $\fml{M}$ is a classifier, and $(\mbf{v},c)$ is an instance.
When describing concepts in explainability, it is to be understood an
underlying explanation problem $\fml{E}$.
Prime implicant (PI) explanations~\cite{darwiche-ijcai18} denote a
minimal set of literals (relating a feature value $x_i$ and a constant
$v_i$ from its domain $\fml{D}_i$) that are sufficient for the
prediction. PI-explanations can be formulated as a problem of
logic-based abduction, and so are also referred to as \emph{abductive
explanations} (AXp)~\cite{inms-aaai19}. 
Formally, given $\mbf{v}=(v_1,\ldots,v_m)\in\mbb{F}$ with
$\kappa(\mbf{v})=c$, an AXp is any minimal subset
$\fml{X}\subseteq\fml{F}$ such that,
\begin{equation} \label{eq:axp}
  \forall(\mbf{x}\in\mbb{F}).
  \left[
    \bigwedge\nolimits_{i\in{\fml{X}}}(x_i=v_i)
    \right]
  \limply(\kappa(\mbf{x})=c)
\end{equation}
%
We associate a predicate $\waxp$ with~\eqref{eq:axp}, such that any
set $\fml{X}\subseteq\fml{F}$ for which $\waxp(\fml{X};\fml{E})$~\footnote{%
Parameterizations are shown after ';', but will be omitted when clear
from context.}
holds is referred to as a \emph{weak} AXp. Thus, every AXp is a weak
AXp that is also subset-minimal.
AXps can be viewed as answering a 'Why?' question, i.e.\ why is some
prediction made given some point in feature space. A different view of
explanations is a contrastive explanation~\cite{miller-aij19}, which
answers a 'Why Not?' question, i.e.\ which features can be changed to
change the prediction. A formal definition of contrastive explanation
is proposed in recent work~\cite{inams-aiia20}.
Given $\mbf{v}=(v_1,\ldots,v_m)\in\mbb{F}$ with $\kappa(\mbf{v})=c$, a
CXp is any minimal subset $\fml{Y}\subseteq\fml{F}$ such that,
\begin{equation} \label{eq:cxp}
  \exists(\mbf{x}\in\mbb{F}).\bigwedge\nolimits_{i\in\fml{F}\setminus\fml{Y}}(x_i=v_i)\land(\kappa(\mbf{x})\not=c) 
\end{equation}
Moreover, we associate a predicate $\wcxp$ with~\eqref{eq:cxp}, such
that any set $\fml{Y}\subseteq\fml{F}$ for which
$\wcxp(\fml{Y};\fml{E})$ holds is referred to as a \emph{weak}
CXp. Thus, every CXp is a weak CXp that is also subset-minimal.
In addition to the predicates $\waxp$ and $\wcxp$, we also introduce
$\axp$ and $\cxp$, which hold true if a given set of features is,
respectively, an AXp or a CXp.

The relationship between explanations and adversarial examples can be
further clarified~\cite{hms-corr23b}, as follows.

\begin{proposition}
  Given an explanation problem $\fml{E}=(\fml{M},(\mbf{v},c))$,
  $\fml{Y}\subseteq\fml{F}$ is a WCXp iff $\fml{M}$ has an \eaex, with
  $l_0$ distance $\epsilon=|\fml{Y}|$.
\end{proposition}

\begin{proof}
  The proof is split into two cases:
  \begin{itemize}[nosep]
  \item[$\Rightarrow$] If there exists one CXp $\fml{Y}$, this means
    that, if we allow the features in $\fml{Y}$ to change value, then
    the prediction changes values. Thus, if we allow
    $\epsilon=|\fml{Y}|$ features to change value, we are guaranteed
    to find an adversarial example with $l_0$ distance $\epsilon$.
  \item[$\Leftarrow$] If there exists an adversarial example with
    $l_0$ distance $\epsilon$, then we construct $\fml{Y}$, with
    $|\fml{Y}|=\epsilon$ by picking the $\epsilon$ features that
    change their value with  respect to $\mbf{v}$ in the adversarial
    example. Thus, $\fml{Y}$ is a WCXp. \qedhere
  \end{itemize}
\end{proof}

Building on the results of R.~Reiter in model-based
diagnosis~\cite{reiter-aij87},~\cite{inams-aiia20} proves a minimal
hitting set (MHS) duality relation between AXps and CXps, i.e.\ AXps
are MHSes of CXps and vice-versa.
Furthermore, it can be shown that both predicates $\waxp$ and $\wcxp$
are monotone. An important consequence of this observation is that
one can then use efficient oracle-based algorithms for finding AXps
and/or CXps~\cite{msjm-aij17}.
Thus, as long as one can devise logic encodings for an ML classifier
(and this is possible for most ML classifiers) and access to a
suitable reasoner, then~\eqref{eq:axp} and~\eqref{eq:cxp} offer a
solution for computing one AXp/CXp. 

Recent years witnessed a rapid development of logic-based
explainability, with practically efficient solutions devised for a
growing number of ML models. Overviews of these results are
available~\cite{msi-aaai22,darwiche-lics23}. 

\paragraph{Distance-restricted explanations.} \label{par:drxps}
With the purpose of relating explanations with robustness, recent
work~\cite{hms-corr23b} introduced the concept of distance-restricted
explanation. 

Given a classifier $\fml{M}$, and instance $(\mbf{v},c)$ and
$\epsilon>0$, a distance-restricted AXp (\eaxp) is a subset-minimal
set of features $\fml{X}\subseteq\fml{F}$ such that,
\begin{equation} \label{eq:dr:waxp}
  \forall(\mbf{x}\in\mbb{F}).
   \left[\bigwedge\nolimits_{i\in\fml{X}}(x_i=v_i)\land
    ||\mbf{x}-\mbf{v}||_p\le\epsilon\right]
  \limply(\kappa(\mbf{x})=c)
\end{equation}
We define distance-restricted CXps (\ecxps)  accordingly. A \ecxp is a
subset-minimal set of features $\fml{Y}\subseteq\fml{F}$, such that,
\begin{equation} \label{eq:dr:wcxp}
  \exists(\mbf{x}\in\mbb{F}).
   \left[\bigwedge\nolimits_{i\in\fml{F}\setminus\fml{Y}}(x_i=v_i)\land
    ||\mbf{x}-\mbf{v}||_p\le\epsilon\land
    (\kappa(\mbf{x})\not=c)\right]
\end{equation}
We extend predicates for (weak) \eaxps/\ecxps to account for the
allowed value of distance as an argument, and parameterized by the
measure of distance, e.g.\ $\waxpe(\fml{X},\epsilon;\fml{E})$.

MHS duality between distance-restricted AXps and CXps has been
proved, which enables the development of algorithms for navigating the
space of \eaxps and \ecxps~\cite{hms-corr23b}.

Finally, there is a simple relationship between (distance-unrestricted
or plain) AXps/CXps and \eaxps/\ecxps. By picking $l_0$, i.e.\ Hamming
distance, and letting $\epsilon=m$, i.e.\ the number of features, then
\eaxps/\ecxps represent \emph{exactly} the (plain) AXps/CXps. Observe
that, by setting $\epsilon=m$, we allow any subset of the features to
be included/excluded from AXps/CXps. Hence, we will be computing
distance-unrestricted AXps/CXps using algorithms developed for
\eaxps/\ecxps.

\subsection{Running Examples} \label{ssec:runex}
%
To motivate the claims in the paper, the following very simple
classifiers are used as the running examples throughout the paper.

\paragraph{First classifier.}
A simple linear classifier is presented in~\cref{ex:runex01}.

\begin{example} \label{ex:runex01}
  A first classifier $\fml{M}_1$ is defined on a single feature
  $\fml{F}_1=\{1\}$,  with $\mbb{D}_{11}=\mbb{R}$. The set of classes
  is $\fml{K}_1=\{0,1\}$, and the training data is given by: 
  $\{ (0.0, 0), \allowbreak (0.3, 0), \allowbreak (0.4, 0), \allowbreak (0.7, 1), \allowbreak (1.0, 1) \}$.
  Furthermore, we use an off-the-shelf ML toolkit,
  e.g.\ scikit-learn~\cite{scikitlearn}, to learn the classifier's
  function $\kappa_1$ as a \emph{linear classifier}
  $\kappa_1:\mbb{D}_{11}\to\fml{K}_1$.
  Accordingly, the model learned by scikit-learn is,
  \[
  \kappa(x_1) = \tn{ITE}(0.93198992\times{x_1}-0.64735516\ge0,1,0)
  \]
  As can be observed, the accuracy of the learned classifier over
  training data is 100\%. 
  Moreover, the question we seek to answer is: \emph{is the classifier
  (locally or globally) robust?}
\end{example}

\paragraph{Second classifier.}
A second classifier, defined in terms of the first classifier is
presented in~\cref{ex:runex01}. 

\begin{example} \label{ex:runex02}
  A second classifier $\fml{M}_2$ is obtained from the first one above
  (see~\cref{ex:runex01}), but defined as follows:
  \[
  \kappa_2(x_1,x_2)=\left\{\begin{array}{lcl}
  \kappa_1(x_1) & \quad & \tn{if $x_1\le1$} \\[5pt]
  \tn{ITE}(x_1>x_2,1,0) & \quad & \tn{otherwise} \\
  \end{array}
  \right.
  \]
  In this case, $\fml{F}_2=\{1,2\}$,
  $\mbb{D}_{21}=\mbb{D}_{22}=\mbb{R}$, $\mbb{F}=\mbb{R}\times\mbb{R}$,
  and  $\fml{K}_2 = \{0,1\}$.
\end{example}

\paragraph{Motivating examples.}
The following examples illustrate some of the definitions in this
section, and motivate some of the results presented in later
sections.

\cref{ex:aex} exemplifies the definition of adversarial examples.

\begin{example} \label{ex:aex}
  For the classifier of~\cref{ex:runex01}, and for the instance
  $(0.7,1)$, from training data (and assuming 100\% accuracy on
  training data), it is apparent that an adversarial example exists by
  setting $\epsilon=0.3$. By manual inspection of the learned model,
  we conclude that we obtain an adversarial example with a smaller
  $\epsilon$, e.g. $\epsilon=0.1$ suffices.
\end{example}

\cref{ex:axp} illustrates the (standard) definitions of AXps and
CXps.

\begin{example} \label{ex:axp}
  Regarding the classifier of~\cref{ex:runex02}, and for the instance
  $((0,1), 0)$, it is apparent that one AXp is $\{1\}$, i.e.\ if
  feature 1 is fixed, then the prediction is guaranteed to be 0.
  \\
  Similarly, one CXp is $\{1\}$, i.e.\ it suffices to allow feature 1
  to change value to be able to change the prediction.
\end{example}

Moreover, \cref{ex:eaxp2,ex:eaxp1} illustrate the definitions of
\eaxps and \ecxps.

\begin{example} \label{ex:eaxp2}
  For the classifier of~\cref{ex:runex02}, we can also investigate
  \eaxps/\ecxps for the instance $((0,1),0)$.
  We consider two example values for $\epsilon$.\\
  For $\epsilon=0.5$, it is clear that the prediction of $\kappa_2$
  remains 0 for any $\mbf{x}$ such that $||\mbf{x}-(0,1)||_0\le0.5$.
  Hence, there are no adversarial examples, and so both the
  \eaxp and the \ecxp are $\emptyset$.\\
  For $\epsilon=0.7$, it is clear that the prediction of $\kappa_2$
  can change to 1 by picking $x_1=0.7$, for example.
  Hence, there exist no adversarial examples obtained by changing the
  value of $x_1$. In this case, both the \eaxp and the \ecxp are
  $\{1\}$. Given~\cref{ex:axp}, it is clear that the same
  \eaxps/\ecxps will be obtained for $\epsilon\ge0.7$.\\
  It should be clear that we analyzed distance-restricted
  explanations for only two possible values of $\epsilon$.
  Another question of interest, which we study later in the paper, is
  to find the value of $\epsilon$ for which the set of \eaxps/\ecxps
  changes.
\end{example}

\begin{example} \label{ex:eaxp1}
  For the classifier of~\cref{ex:runex01}, explanations are apparently
  less interesting, because we have a single feature.
  Nevertheless, the arguments of~\cref{ex:eaxp2} also hold for
  $\kappa_1$.\\
  For $\epsilon=0.5$, both the \eaxp and the \ecxp are $\emptyset$.\\
  For $\epsilon=0.7$, both the \eaxp and the \ecxp are $\{1\}$. In
  addition, it is clear that the same \eaxps/\ecxps will be obtained
  for $\epsilon\ge0.7$.
\end{example}

%% file: relw.tex
\section{Related Work} \label{sec:relw}

%
The realization that ML models are most often
brittle~\cite{biggio-ecml13,goodfellow-iclr14,goodfellow-iclr15},
i.e.\ that ML models can exhibit adversarial examples, motivated a
massive body of research over the last decade on deciding the
robustness of ML models.
The goal of this section is to briefly overview works that are of
special interest to the paper's topics, especially ML model robustness
and the identification of adversarial examples.
Moreover, a growing number of
surveys~\cite{berker-corr19,li-ieee-tnnls20,wang-jbd20,chakraborty-tit21,rokach-acmcs22,liang-electronics22,mao-acmcs22,yu-acmcs23,guan-acmcs23}
illustrate the importance of robustness and adversarial examples for
the practical deployment of ML models.

\paragraph{Local robustness.}
%
%
The definition of local robustness proposed in most of past work
matches the one used in this paper (see~\eqref{eq:locrob3}).
Examples of tools that decide local robustness are those evaluated in
VNN-COMP~\cite{johnson-sttt23}.

\jnoteF{Cite papers on LOCAL ROBUSTNESS.}

\paragraph{Global robustness.}
%
%
Past work considers global robustness as proposed
in~\cref{eq:globrob}, which formalizes~\cref{def:certrob}. This is
the case with~\cite{Seshia-atva18,Wagner-ccs21}, but
also~\cite{kolter-icml19a,kolter-icml20a,kolter-iclr20a,kolter-nips21,kolter-iclr23a}.
Some works propose a slightly modified definition of global
robustness~\cite{barrett-cav17}:
\begin{equation} \label{eq:globrobv}
  \forall(\mbf{v},\mbf{x}\in\mbb{F}).\left(||\mbf{x}-\mbf{v}||_p\le\epsilon\right)\limply\left(|\kappa(\mbf{x})-\kappa(\mbf{v})|\le\delta\right) 
\end{equation}
where $\delta>0$.
When compared with~\eqref{eq:globrob}, the modified definition targets
neural networks, especially when these compute real-valued outputs.
A number of works adopt this definition of global
robustness~\cite{WangHZ22,WangHZ-corr22,fu2022reglo},
but~\cite{fu2022reglo} imposes no constraint on $\mbf{x}$ and
$\mbf{v}$.
It should be noted that this definition is not without problems. For
ML classifiers, e.g.\ image classification, conditions on the values
of the outputs are uninteresting, and so a definition similar
to~\eqref{eq:globrob} must be considered.

A different approach is adopted in~\cite{Kwiatkowska-ijcai19} where
global robustness is defined with respect to a \emph{finite set} of
points in feature space, and not \emph{all} points in feature space.
Yet another take on global robustness is to reject inputs that
are classified as AEx~\cite{LeinoWF21,Kolter-aistats23}. Thus a model
can return class \emph{abstain} on a given input $\mbf{x}$.  
Finally, another line of research is represented by DeepSafe, which
finds safe regions where robustness is guaranteed~\cite{barret-atva18,Dimitrov0GV22}.

\paragraph{Robustness certification.}
%
%
Work on certifying robustness can bee traced to 
\cite{kolter-icml19a,proven-icml19,GehrMDTCV18,SinghGMPV18}   and 
more recently~\cite{kolter-icml20a,kolter-iclr20a,kolter-nips21,kolter-iclr23a} that 
leverage on local robustness property to provide certification and/or quantify  
robustness of models against adversarial examples. 
As illustration, empirical evaluation reported in~\cite{kolter-iclr23a} considers  
a collection of  100,000 and 10,000 samples, respectively,  drawn from CIFAR10 and 
ImageNet datasets, that serve to  certify {\it robustness accuracy}, i.e. 
percentage of samples failed/succeeded  in the local robustness test. 
Besides, works reported in~\cite{liu2020certified,WangHZ-corr22,WangHZ22} adopt global 
robustness property to certify whether or not the analyzed model is robust.
Furthermore, some works \cite{fu2022reglo,WangHZ22} use local and global 
robustness techniques to measure lower and upper bounds for robustness. 
Another recent work \cite{Dimitrov0GV22} computes regions  
with robustness certification on all possible points in these regions.

%% file: cons.tex
\section{The Cons of Robustness}
\label{sec:cons}

This section proves a number of negative results regarding global and
local robustness, but also regarding robustness certification. More
importantly, those negative results impact the conclusions drawn in
earlier work on robustness.
Nevertheless, this section also discusses ways to cope with these
negative results.

\subsection{Basic Negative Results -- Real-Valued Domains}
\label{ssec:cons:rd:neg}

\paragraph{There is no global robustness.}
A straightforward observation is that, given the proposed definition
of global robustness (see~\eqref{eq:globrob}), then there exist no
non-trivial globally robust classifiers.

\begin{proposition} \label{prop:noglobrob}
  Any non-trivial classifier defined on real-valued features is not
  globally robust, independently of the value of $\epsilon$ chosen.
  (We assume that the measure of distance considered is $l_p$, with
  $p\ge1$.)
\end{proposition}

\begin{proof}
  We consider a classifier with two classes $\{a,b\}$, with
  $a\not=b$. Let $\mbf{v}_a,\mbf{v}_b\in\mbb{F}$ be such that
  $\kappa(\mbf{v}_a)=a$ and $\kappa(\mbf{v}_b)=b$.
  Moreover, let $\mbf{z}\in\mbb{F}$ be such that
  $\mbf{z}=\mbf{v}_a+\delta\times(\mbf{v}_b-\mbf{v}_a)$, with 
  $0\le\delta\le1$. 
  We pick $\mbf{z}$ such that with $0\le\eta\le\delta$ and
  $0\le\delta+\eta\le1$,
  we can define 
  $\mbf{x}=\mbf{v}_a+(\delta-\eta)\times(\mbf{v}_b-\mbf{v}_a)$
  and
  $\mbf{y}=\mbf{v}_a+(\delta+\eta)\times(\mbf{v}_b-\mbf{v}_a)$,
  such that $\kappa(\mbf{x})=a$ and $\kappa(\mbf{y})=b$.
  We claim that such a point $\mbf{z}$ must exist; otherwise, the
  classifier would be constant between $\mbf{v}_a$ and $\mbf{v}_b$,
  and that is impossible because we impose $a\not=b$.
  Thus for $\epsilon>0$, it is the case that the classifier is not
  locally robust on point $\mbf{z}$, since the prediction
  changes within a $l_p$ ball of $\mbf{z}$. Since the classifier is
  not locally robust on $\mbf{z}$, then it is not globally robust.
\end{proof}

An observation that mimics~\cref{prop:noglobrob} is made in
recent work~\cite{LeinoWF21}. However, the consequences of such
observation were not investigated further. Instead, a proposed
solution to the problem of global robustness was to change the
training of the classifier to return an indication of inputs that
cannot be guaranteed to be robust. Hence, the solution proposed is to
move from \emph{a posteriori} deciding robustness to training for
robustness.

\begin{example}
  With respect to~\cref{ex:runex01}, recall that the model learned by
  scikit-learn is,
  \[
  \kappa(x_1) = \tn{ITE}(0.93198992\times{x_1}-0.64735516\ge0,1,0)
  \]
  Clearly, the value of $x_1$ for which the predicted class
  transitions from 0 to 1 is 0.69459459.
  Hence, exhibiting point $x_1=0.69459459$ is a proof that the model
  is \emph{not} globally robust.
\end{example}

\paragraph{No global robustness implies no local robustness.}
One observation that stems from the proof of~\cref{prop:noglobrob}
is that deciding local robustness is also problematic.
Concretely, if we are allowed to select a suitable point in feature
space, then it is the case that,

\begin{proposition} \label{prop:nolocrob}
  For any non-trivial classifier defined on real-valued features,
  there exists a point for which the classifier is not locally
  robust, independently of the value of $\epsilon$ chosen.
\end{proposition}

\begin{proof}
  The same arguments used for the proof~\cref{prop:noglobrob} apply.
\end{proof}

\begin{example} \label{ex:runex:nolocrob}
  With respect to~\cref{ex:runex01}, if we sample training data, then
  the model will be declared locally robust for $\epsilon<0.00540541$,
  due to point $x_1=0.7$. Clearly, if we allow complete freedom on
  which values of $x_1$ to sample, we will then conclude that, for
  $x_1=0.69459459$, robustness is non-existing for any value of
  $\epsilon>0$, and so the classifier is not robust.
\end{example}

Observe that what~\cref{prop:nolocrob} claims is that one can always
find points in feature space for which local robustness does not
hold. Thus, claiming local robustness based on successful proving
local robustness for some selected points in feature space, does not
equate with global robustness holding in all points in feature space.

\paragraph{The robustness of non-trivial classifiers cannot be certified.}
In recent years, a number of works have studied robustness
certification~\cite{kolter-icml19a,kolter-iclr20a,kolter-icml20a,kolter-nips21,kolter-iclr23a}.
Using the definition of robustness certification proposed
in~\cref{sec:prelim} (see~\cref{def:certrob}),
which is taken verbatim from~\cite{kolter-icml19a},
then we can claim that,

\begin{proposition} \label{prop:norobcert}
  For any non-trivial classifier defined on real-valued features,
  robustness cannot be certified, independently of the value of
  $\epsilon$ chosen.
\end{proposition}

\begin{proof}
  Under the stated assumptions, and from the proofs
  of~\cref{prop:nolocrob,prop:noglobrob}, for any classifier we can
  construct a point for which a classifier is not locally robust.
  Hence, robustness cannot be certified.
\end{proof}

\begin{example}
  The analysis of~\cref{ex:runex:nolocrob} also demonstrates that
  robustness certification will fail for the example classifier, for
  any $\epsilon>0$, as long as the point $x_1=0.69459459$ is analyzed.
  In contrast, if the training data were to be sampled for local
  robustness, then one would declare robustness to be certified for
  $\epsilon<0.00540541$. Evidently, such conclusion would be in
  error.
\end{example}

\paragraph{Counterexamples to local robustness.}
Under the assumption of unconstrained inputs, we proved earlier in
this section that no ML model is (locally) robust.
As a result, and given some $\epsilon>0$, counterexamples to
(unconstrained) local robustness are guaranteed to exist, and can be
obtained from any pair of points $\mbf{v},\mbf{x}$ in feature space
such that,
\begin{equation} \label{eq:nolr}
  \exists(\mbf{x},\mbf{v}\in\mbb{F}).%
  \left(||\mbf{x}-\mbf{v}||_{p}\le\epsilon\right)\land\left(\kappa(\mbf{x})\not=\kappa(\mbf{v})\right)
\end{equation}
Observe that requiring $(\mbf{x}\not=\mbf{v})$ is unnecessary.
Also, we note that~\eqref{eq:nolr} is just the negation
of~\eqref{eq:globrob}, i.e.\ the condition for global robustness.
In addition, it is easy to accommodate the cases where a classifier can
flag points as not being robust~\cite{LeinoWF21}. (Earlier work
proposes that such points be flagged with a different class $\perp$,
but in this paper we use $\oslash$ instead.)
Finally, it is also plain to formulate the search for counterexamples
as a decision problem. Given some logic formula $\varphi$ and a target
logic theory $\fml{T}$, then $\llbracket\varphi\rrbracket_{\fml{T}}$
represents the encoding of $\varphi$ in the target logic theory
$\fml{T}$. Also, let $\SAT_{\fml{T}}$ represent a reasoner for theory
$\fml{T}$. Then, we use~\eqref{eq:nolr} to formulate the decision
problem in theory~$\fml{T}$ as follows:
\begin{equation} \label{eq:nolr:dec}
  \SAT_{\fml{T}}\left(\llbracket\left(||\mbf{x}-\mbf{v}||_{p}\le\epsilon\right)\land\left(\kappa(\mbf{x})\not=\kappa(\mbf{v})\right)\rrbracket_{\fml{T}}\right)
\end{equation}
\cref{sec:res} presents results for Binarized Neural Networks (BNNs)
when $\fml{T}$ denotes propositional logic (and so $\SAT_{\fml{T}}$
corresponds to a call to a boolean satisfiability (SAT) reasoner.

\subsection{Basic Negative Results -- Discrete Domains}
\label{ssec:cons:dd:neg}

\jnoteF{ToDo:\\
  Categorical features \& Hamming distance: required to have $\epsilon\ge1$ \\
  Integer-valued features: required to have $\epsilon\ge1$ \\
  Discretized features: required to have $\epsilon\ge\eta$, where
  $\eta$ is the discretization step. \\
}

The negative results in~\cref{ssec:cons:rd:neg} also hold in the case
of classifiers with discrete features. However, the values of
$\epsilon$ considered are further constrained.
First, we consider a classifier with categorical features, an $l_0$
norm and unconstrained inputs. In this case, it is guaranteed that one
cannot have adversarial examples if $\epsilon<1$, i.e.\ if one
prevents any feature from changing value. If we impose the constraint 
$\epsilon\ge1$, then the results of~\cref{ssec:cons:rd:neg} hold,
i) no non-trivial classifier is globally robust; ii) for any
non-trivial classifier there are points in feature space that are not
locally robust; and iii) robustness cannot be certified.

There are also settings where features are discrete and result from 
discretizing real-valued features. This is the case for example when
using binarized neural networks (BNNs)~\cite{bengio-nips16} (see also
the experiments in~\cref{sec:res}).
In these cases we assume a discretization step $\delta$.
As a result, the comments made for classifiers with categorical
features also apply in this case, with the constraint that
$\epsilon\ge\delta$, i.e.\ the smallest distance to consider is no
less than the discretization step.

For example, for the experiments of \cref{sec:res}, when the $l_0$
norm is used, it is the case that $\delta=1$.
As a result, in the experiments we used $\epsilon=1$, thus targeting
the \emph{smallest} distance that could possibly be considered. As the
results show, and as it should be expected, one can find
counterexamples to local/global robustness for all the experiments.

\subsection{Practical Consequences}

The results in~\cref{ssec:cons:rd:neg} have important practical
consequences. First, the experimental setup most often used in the
assessment of local robustness exhibits critical shortcomings.
In a significant body of earlier work, local robustness is assessed by
randomly sampling feature space. This is the case with evaluations of
local robustness in~\cite{johnson-sttt23}.
%
%
This means that, either sampling does not pick the right points in
feature space, and so one is allowed to (incorrectly) decide for local
robustness, or sampling picks one of the points that must exist, and
so (the expected) non local robustness is decided.
One paradigmatic example is VNN-COMP\cite{johnson-sttt23}, where
uniform random sampling has been employed in all the previous
competitions. The bottom line is that any classifier declared local
robust is guaranteed \emph{not} to be so.

The same remarks apply in the case of global robustness, but in this
case the number of existing works is a fraction of the number of works
on deciding local robustness.

\paragraph{Examples of inadequate robustness assessment.}
One example of the limitations of randomly sampling for assessing
robustness is VNN-COMP~\cite{johnson-sttt23}. The description of the
most recent competitions~\cite{johnson-corr22a,johnson-corr23a}
confirms that robustness is assessed by randomly sampling existing
datasets.
In these cases, any results indicating that NNs to being robust are
necessarily incorrect.
Furthermore, past works claiming robustness certification are
inaccurate. Methods based on automated reasoners, which assume
unconstrained inputs, will be in error if ML models are declared
robust. Distribution-restricted methods can only provide probabilistic
guarantees. In such cases, one must trust that the inferred input
distribution faithfully captures possible inputs to the ML model.
More importantly, one must also trust that the sampling methods used
will offer enough rigor.

Besides the shortcomings of local robustness and certification, the
limitations of existing methods for attaining global robustness are
also clear.

\paragraph{Practical assessment of shortcomings.}
As noted earlier~\eqref{eq:nolr:dec}, in practice it is
conceptually simple to demonstrate the limitations of local robustness
using a dedicated automated reasoner.
It suffices to decide the existence of a point $\mbf{z}$ in feature
space which, for arbitrarily small $\epsilon$, it is the case that
there exists a point in the $\epsilon$ ball corresponding to a
prediction other than $\kappa(\mbf{z})$.
\cref{sec:res} summarizes results illustrating not only the existence
of such points for complex classifiers, but also the practical
scalability of this approach.

\paragraph{Existing solutions.}
%
Some approaches avoid the problems reported in this section by curbing
the claims about (certified) robustness. For
example, the definition of global robustness in some
works~\cite{Kwiatkowska-ijcai19} considers a finite set of points
where local robustness is assessed. As long as the inputs respect such
a finite set of points, then the non-existence of adversarial examples
for robust ML models is guaranteed. Unfortunately, if the inputs were
to be known, then the need for ML models would be non-existing.
In a similar vein, other tools~\cite{barret-atva18} ensure robustness
in specific regions of feature space. Although more general
that~\cite{Kwiatkowska-ijcai19}, similar limitations apply.

Some works exploit sampling of the inputs according to some inferred
input
distribution~\cite{kolter-icml19a,kolter-icml20a,kolter-iclr20a,kolter-nips21,kolter-iclr23a}.
Such works offer no formal guarantees of rigor, and so
claims of \emph{certification} hold probabilistically assuming that all
possible inputs respect the assumed input distribution.
%

%
%

\subsection{Threats to Validity \& Discussion}

One possible criticism to the results in this section is that some
works on robustness 
do not assume unconstrained inputs, but instead sample the inputs
according to the distribution inferred from training data or specific
imposed distribution~\cite{proven-icml19,kolter-icml19a}.
%
%
In such a situation, one can argue that both local and  global
robustness can still be safely decided.
Clearly, the claims of rigor differ in the two cases. As the paper
shows, approaches based on automated reasoners that consider all
possible inputs in feature space cannot provide guarantees of
robustness. However, the same applies to any other approach when one
seeks absolute guarantees of robustness.
In addition, there is recent work showing the importance of sampling
out of the
distribution~\cite{YinLSCG19,Hendrycks-acl20,Hendrycks-iccv21,Finn-nips22},
but this again exhibits the shortcomings of local/global robustness
identified earlier in this section.

Furthermore, and regarding the sampling of inputs according to some
inferred input distribution is not without limitations. First,
inferring an input distribution from a negligible fraction of feature
space can be a source of error. Second, an input distribution declares 
inputs more or less likely, but does not prevent the possibility, no
matter how negligible, of inputs not in the distribution. Third,
robustness tools based on automated reasoners assume that \emph{all}
inputs are possible, i.e.\ unconstrained inputs are assumed, when
deciding local robustness for a concrete point in feature
space~\cite{barrett-cav17}.
If input distributions
were to be taken into account for such tools when deciding the
points in feature space to analyze,
then even for deciding local
robustness for a single point in feature space, input distributions
would have to be accounted for.

To the best of our knowledge, there is no clear way on how to
instrument automated reasoners to account for input distributions.
The key point is that sampling according to the input distribution
could be a source of error, and so the integration with tools based on
automated reasoners is unclear. 


A different approach is to explicitly assume that some inputs are not
accepted. For example, allowing an elderly person to be of a fairly
young age.
To the best of our knowledge, part work on robustness does not account
for disallowed inputs. In contrast, the topic has been researched in
formal explainability~\cite{rubin-aaai22,ignatiev-aaai23},
by considering constraints on the inputs.
Future work may re-analyze robustness in light of such constraints.

%% file: pros.tex
\section{Some Pros of Robustness}
\label{sec:pros}

Under a scenario in which all inputs are possible, \cref{sec:cons}
raises important concerns about the usefulness of attempting to prove
robustness, be it local or global.
This might seem to suggest the ongoing efforts towards devising
efficient robustness tools are pointless.
This section argues otherwise, and discusses emerging significant uses
of robustness tools.

\paragraph{From adversarial examples to \beaxps/\becxps.}
Recent work revealed a tight relationship between formal explanations
and the non-existence of (constrained) adversarial
examples~\cite{hms-corr23b}, enabled by the concept of
distance-restricted AXps/CXps (see definition
in~\cpageref{par:drxps}). Furthermore, MHS duality between
distance-restricted AXps and CXps enable the navigation of the space
of AXps and CXps.
More importantly, the computation of distance-restricted AXps/CXps can
be instrumented using tools for finding adversarial examples, as long
as those tools respect a few simple conditions. 
As a result, different algorithms have been
proposed for computing explanations~\cite{barrett-corr22,hms-corr23b},
which explore off-the-shelf tools for deciding the existence of
adversarial examples.

Throughout this section, we assume that the existence of adversarial
examples is decided by calls to a suitable oracle. In the algorithms
described in this section, 
this oracle is represented by a predicate
$\robt(\epsilon,\fml{Q};\fml{E},p)$, parameterized by the explanation
problem $\fml{E}=(\fml{M},(\mbf{v},x))$, where $\fml{M}$ is the
classifier, and by the norm $\l_p$, and with arguments $\epsilon>0$
and $\fml{Q}\subseteq\fml{F}$.
Furthermore, the predicate $\robt(\epsilon,\fml{Q};\fml{E},p)$ holds
true if the ML classifier $\fml{M}$ exhibits an adversarial example
within distance $\epsilon$, with the features in $\fml{Q}$ fixed to
the values dictated by $\mbf{v}$.
As a result, the conditions that tools for finding adversarial
examples must respect when finding adversarial examples is to allow
for some features to be fixed to  values determined by the instance
$(\mbf{v},c)$.

\begin{algorithm}[t]
  \input{./algs/findaxp_del}
  \caption{Basic algorithm to find AXp using AE search}
  \label{alg:axp:del}
\end{algorithm}

\begin{algorithm}[t]
  \input{./algs/findcxp_del}
  \caption{Basic algorithm to find CXp using AE search}
  \label{alg:cxp:del}
\end{algorithm}

To illustrate the relationships between explanations and adversarial
examples, \cref{alg:axp:del,alg:cxp:del} summarize two simple
algorithms for computing \eaxps/ \ecxps. One argument is the value of
$\epsilon$, whereas the other argument is a set of features that must
be kept fixed (for AXps) or set free (for CXps). This argument is
important for devising algorithms for the enumeration of explanations;
for computing a single explanation, we set $\fml{R}=\fml{F}$.
For both~\cref{alg:axp:del,alg:cxp:del} the loop invariant is that
$\fml{S}$ is, respectively, a weak \eaxp/\ecxp.
In the case of an \eaxp, the features in $\fml{S}$, if fixed, cause
that no adversarial example can be identified (given the distance
$\epsilon$).
In the case of an \ecxp, the features in $\fml{S}$, if fixed, cause
that no adversarial example can be identified (given the distance
$\epsilon$).
Features are freed/fixed as long as the loop invariant is preserved.
An oracle for deciding the existence of an adversarial example is used
to decide whether or not the loop invariant is preserved when another
feature is freed/fixed.
It can be shown that the final set $\fml{S}$ is a \eaxp/\ecxp.
The proposed algorithms are among the simplest that can be
devised%
\footnote{%
The problem of finding one \eaxp/\ecxp represents an example of
computing one minimal set subject to a monotone predicate
(MSMP)~\cite{msjm-aij17}, and so existing algorithms for MSMP can be
used.}.
Nevertheless, other more sophisticated algorithms for deciding the
existence of adversarial examples can be used, which require in
practice fewer calls to the oracle than the algorithms outlined in the
paper.

\paragraph{Navigating the space of \beaxps/\becxps.}
%
Besides computing one distance-restricted explanation, one may be
interested in navigating the sets of distance-restricted
explanations. For example, we may be interested in deciding whether a
sensitive feature can occur in some explanation, or we may just be
interested in finding some other explanation when the reported one is
uninteresting.
\begin{algorithm}[t]
  \input{./algs/allxp}
  \caption{MARCO-like~\cite{lpmms-cj16} enumeration of AXps/CXps}
  \label{alg:xp:all}
\end{algorithm}
\cref{alg:xp:all} details a MARCO-like~\cite{lpmms-cj16} approach for
enumerating both all AXps and CXps. This algorithm illustrates the
integration of Boolean Satisfiability (SAT) algorithms with tools for
finding adversarial examples, with the purpose of computing distance
restricted AXps and CXps.

\paragraph{New insights into explanations.}
%
For non-trivial classifiers defined on real-values features, the fact
that local robustness does not hold on all points of feature space
reveals not only the guaranteed existence of adversarial examples, but
it also reveals new properties about explanations. We discuss one such
example.

\begin{proposition}
  For a non-trivial classifier defined on real-valued features, and
  for a measure $l_p$, there exist instances for which, for any
  $\epsilon>0$, there exist \eaxps and \ecxps.
\end{proposition}

\begin{Proof}(Sketch)
  We pick a point $\mbf{t}$ in feature space for which for any
  $\epsilon>0$, there exists an adversarial example. It is immediate
  that, for each   adversarial example, we can construct a CXp, in
  that one simply adds to the CXp the features that change value in
  the adversarial example. The same applies for \ecxps.
  The next step is to enumerate all such \ecxps. From these, by
  hitting set duality, we enumerate the \eaxps.
  This proves the result
\end{Proof}

%% file: algs/findaxp_del.tex
\begin{flushleft}
  \hspace*{\algorithmicindent}
  \textbf{Input}: {
    Arguments: 
    $\fml{R}$, $\epsilon$;
    Parameters: 
    $\fml{E}$,
    $p$}\\
  \hspace*{\algorithmicindent}
  \textbf{Output}: {One AXp $\fml{S}$}
\end{flushleft}

\begin{algorithmic}[1]
  \Function{$\findaxpdel$}{$\fml{R},\epsilon;\fml{E},p$}
  \Comment{Inv:~$\neg\robt(\epsilon,\fml{R})$}
  \State{$\fml{S}\gets\fml{R}$}
  \Comment{Initially, no feature in reference set is free} 
  \For{$i\in\fml{R}$}
  \State{$\fml{S}\gets\fml{S}\setminus\{i\}$}
  \Comment{Free feature $i$}
  \State{$\outc=\robt(\epsilon,\fml{S};\fml{E},p)$}
  \If{$\outc$}
  \Comment{If invariant not preserved}
  \State{$\fml{S}\gets\fml{S}\cup\{i\}$}
  \Comment{Then, fix again feature $i$}
  \EndIf
  \EndFor
  \State{\tbf{return} $\fml{S}$}
  \Comment{$\fml{S}$ suffices for picking the prediction}
  \EndFunction
\end{algorithmic}

%% file: algs/findcxp_del.tex
\begin{flushleft}
  \hspace*{\algorithmicindent}
  \textbf{Input}: {
    Arguments: 
    $\fml{R}$, $\epsilon$;
    Parameters: 
    $\fml{E}$,
    $p$}\\
  \hspace*{\algorithmicindent}
  \textbf{Output}: {One CXp $\fml{S}$}
\end{flushleft}

\begin{algorithmic}[1]
  \Function{$\findcxpdel$}{$\fml{R},\epsilon;\fml{E},p$}
  \Comment{Inv:~$\robt(\epsilon,\fml{R})$}
  \State{$\fml{S}\gets\fml{R}$}
  \Comment{Initially, no feature in reference set is fixed}
  \For{$i\in\fml{R}$}
  \State{$\fml{S}\gets\fml{S}\setminus\{i\}$}
  \Comment{Fix feature $i$}
  \State{$\outc=\robt(\epsilon,\fml{F}\setminus\fml{S};\fml{E},p)$}
  \If{\tbf{not}~$\outc$}
  \Comment{If invariant not preserved}
  \State{$\fml{S}\gets\fml{S}\cup\{i\}$}
  \Comment{Then, free again feature $i$}
  \EndIf
  \EndFor
  \State{\tbf{return} $\fml{S}$}
  \Comment{$\fml{S}$ suffices for changing the prediction}
  \EndFunction
\end{algorithmic}

%% file: algs/allxp.tex
\begin{flushleft}
  \hspace*{\algorithmicindent}
  \textbf{Input}: {
    Argument $\epsilon$,
    Parameters $\fml{E}$, $p$
  } 
  \hspace*{\algorithmicindent}
  %
\end{flushleft}

\begin{algorithmic}[1]
  \State{%
    \label{alg:exp:ln01}$\fml{H}\gets\emptyset$}%
  \Comment{$\fml{H}$ defined on set $U=\{u_1,\ldots,u_m\}$}
  \Repeat\label{alg:exp:ln02} 
  \State{%
    \label{alg:exp:ln03}$(\outc,\mbf{u})\gets\SAT(\fml{H})$}
  \If{\label{alg:exp:ln04}$\outc=\TRUE$}
  \State{\label{alg:exp:ln05}$\fml{S}\gets\{i\in\fml{F}\,|\,u_i=0\}$}%
  \Comment{$\fml{S}$: \emph{fixed} features}
  \State{\label{alg:exp:ln06}$\fml{U}\gets\{i\in\fml{F}\,|\,u_i=1\}$}%
  \Comment{$\fml{U}$: \emph{universal} features} 
  \If{\label{alg:exp:ln07}$\wcxp(\fml{U},\epsilon;\fml{E})$}
  \Comment{$\fml{U}=\fml{F}\setminus\fml{S}\supseteq$ some CXp}
  \State{\label{alg:exp:ln08}$\fml{P}\gets\findcxpdel(\fml{U},\epsilon;\fml{E})$}
  \State{\label{alg:exp:ln09}$\prtcxp(\fml{P})$}
  \State{\label{alg:exp:ln10}$\fml{H}\gets\fml{H}\cup\{(\lor_{i\in\fml{P}}\neg{u_i})\}$}
  \Else\label{alg:exp:ln11}
  \Comment{$\fml{S}\supseteq$ some AXp}
  \State{\label{alg:exp:ln12}$\fml{P}\gets\findaxpdel(\fml{S},\epsilon;\fml{E},p)$}
  \State{\label{alg:exp:ln13}$\prtaxp(\fml{P})$}
  \State{\label{alg:exp:ln14}$\fml{H}\gets\fml{H}\cup\{(\lor_{i\in\fml{P}}{u_i})\}$}
  \EndIf
  \EndIf
  \Until{\label{alg:exp:ln15}$\outc=\FALSE$}
\end{algorithmic}

%% file: res.tex
\section{Experimental Evidence} \label{sec:res}
This section presents a summary of practical evidence 
of our results on global robustness (and so indirectly on the
impossibility of global local robustness) for the case study of
Binarized Neural Networks (BNNs) classifiers trained with image
datasets. 
The decision to opt for BNNs results exclusively from the better
scalability of its dedicated reasoners when compared with the ML
classifiers, e.g. feedforward Neural Networks (NNs).

\paragraph{Experimental setup.}
The experiments are conducted on a MacBook Pro with a Dual-Core Intel
Core~i5 2.3GHz CPU with 8GByte RAM running macOS Ventura.
The time limit is set 600 s, whilst no imposition for the memory limit.

\paragraph{Benchmarks.}
The assessment is performed on  image MNIST digits~\cite{LeCun98} and 
CIFAR10~\cite{madry19-cifar10} datasets.
We considered the feedforward BNN models used in~\cite{rinard-nips20} 
that are trained with MNIST and CIFAR10 datasets with varying different 
training parameters. Hence, in total we collected 13 BNNs, i.e. 8 MNIST 
and 5 CIFAR10 BNNs.

\paragraph{Prototypes implementation.}
We implemented a formal global robustness verifier for BNNs on top of
the Python prototype of~\cite{rinard-nips20} designed for local robustness
checking. 
Concretely, we encode two copies of the BNN, one 
 replica that represents  $\kappa(\mbf{x})$ and 
another one for $\kappa(\mbf{y})$ and then enforce them 
to pick different classes, i.e.  $\kappa(\mbf{x}) \neq \kappa(\mbf{y})$.
Moreover, we used PySAT~\cite{imms-sat18} to generate BNN SAT-based encoding, 
also used the Python package PyPBLib~\cite{pypblib} of 
PBLib~\cite{pblib.sat2015},  integrated in PySAT,  
to generate the pseudo-Boolean encoding.
%
%
Moreover, Roundingsat  oracle\footnote{%
Roundingsat is a pseudo-Boolean solver powered with  
an LP solver SoPlex~\cite{soplex}.}~\cite{Nordstrom-ijcai18}  is 
instrumented to solve the pseudo-Boolean formulation of the robustness 
targeted problem.
We underline that in our observations on preliminary results. as well
as pointed out in~\cite{rinard-nips20}, pseudo-Boolean encoding 
and Roundingsat oracle shown (slightly) better performances than 
SAT solvers. 
As a result, we solely report the experimental results of the pseudo-Boolean 
encodings of the selected BNN benchmarks.

%
\input{./figs/BNNs-img}

\paragraph{Results.}
\autoref{tab:BNNs-pb} summarizes  the results on BNNs containing 
4804 to 6244 neurons, 3 to 4 hidden layers (i.e. tensor shape 
of cifar10 (resp.\ mnist) networks is  (3072, 4096, 2048, 100, 10) 
(resp.\ (784, 500, 300, 200, 100, 10)) ).  
Columns {\bf m} and {\bf K} report, respectively, the
number of features and classes in the dataset.
Columns {\bf D}, {\bf\#N}  show, respectively, number 
of hidden layers, total number of neurons 
of the BNN classifier.
Columns  {\bf p}, {\bf AEx} and  {\bf Time}  report, resp., the fixed hamming or 
chebyshev distance value,  whether or not the model is robust under  
{\bf $\epsilon$} and  the average  runtime  for checking robustness.
(Note that the quantization step $qs$ is fixed to 0.61 for MNIST and 0.064 
for CIFAR10, so $\lfloor \frac{\epsilon}{qs}\rceil = 1$ for MNIST (resp.\  CIFAR10).)
As can be observed from the results, our robustness queries are 
able to identify adversarial examples for all tested BNNs, with both 
$l_0$ and $l_\infty$ distances, and in just a few seconds.
Clearly, our results confirm the theoretical findings presented earlier 
that global robustness query will always report an adversarial example 
for any non-trivial classifier. 
%


%% file: figs/BNNs-img.tex
%
\sisetup{parse-numbers=false,detect-all,mode=text}
\setlength{\tabcolsep}{5pt}
\let\lpr\undefined
\let\rpr\undefined
\newcommand{\lpr}{(}
\newcommand{\rpr}{)}

\begin{table}[ht]
  \centering
\resizebox{\textwidth}{!}{
  \begin{tabular}{l>{\lpr}S[table-format=3.0,table-space-text-pre=\lpr]S[table-format=3.0,table-space-text-post=\rpr]<{\rpr}
  S[table-format=2]S[table-format=5.0]
  S[table-format=1]S[table-format=1]S[table-format=3.2] 
  S[table-format=1.2]S[table-format=1]S[table-format=2.2] }
\toprule[1.2pt]
\multirow{2}{*}{\bf Dataset} & \multicolumn{2}{c}{\multirow{2}{*}{\bf (m,~K)}}  & \multicolumn{2}{c}{\bf BNNs} & \multicolumn{3}{c}{\bf  $l_0$ distance } & \multicolumn{3}{c}{\bf  $l_\infty$ distance } \\
  \cmidrule[0.8pt](lr{.75em}){4-5}
  \cmidrule[0.8pt](lr{.75em}){6-8}
  \cmidrule[0.8pt](lr{.75em}){9-11}
& \multicolumn{2}{c}{} & {\bf D}  & {\bf \#N}  &  { $\bm\epsilon$ }  &  {\bf AEx} & {\bf Time }  & { $\bm\epsilon$ }  &  {\bf AEx} & {\bf Time }  \\
\toprule[1.2pt]

cifar10-s-adv2 & 3072 & 10 & 3 & 6244 & 1 & \checkmark & 5.67 & 0.07 & \checkmark & 1.04 \\
cifar10-s-adv8 & 3072 & 10 & 3 & 6244 & 1 & \checkmark & 2.40 & 0.07 & \checkmark & 0.99 \\
cifar10-s-adv8-ternweight & 3072 & 10 & 3 & 6244 & 1 & \checkmark & 17.89 & 0.07 & \checkmark & 33.61 \\
cifar10-s-advnone & 3072 & 10 & 3 & 6244 & 1 & \checkmark & 1.56 & 0.07 & \checkmark & 1.38 \\
cifar10-s-advnone-ternweight & 3072 & 10 & 3 & 6244 & 1 & \checkmark & 158.23 & 0.07 & \checkmark & 2.66 \\
mnist-mlp & 784 & 10 & 4 & 1100 & 1 & \checkmark & 0.81 & 0.62 & \checkmark & 0.19 \\
mnist-s-adv0.1 & 784 & 10 & 3 & 4804 & 1 & \checkmark & 3.14 & 0.62 & \checkmark & 0.17 \\
mnist-s-adv0.3 & 784 & 10 & 3 & 4804 & 1 & \checkmark & 0.59 & 0.62 & \checkmark & 0.13 \\
mnist-s-adv0.3-hardtanh & 784 & 10 & 3 & 4804 & 1 & \checkmark & 0.24 & 0.62 & \checkmark & 0.17 \\
mnist-s-adv0.3-ternweight-wd0 & 784 & 10 & 3 & 4804 & 1 & \checkmark & 31.13 & 0.62 & \checkmark & 1.32 \\
mnist-s-adv0.3-ternweight-wd1 & 784 & 10 & 3 & 4804 & 1 & \checkmark & 2.00 & 0.62 & \checkmark & 0.35 \\
mnist-s-advnone & 784 & 10 & 3 & 4804 & 1 & \checkmark & 0.36 & 0.62 & \checkmark & 0.22 \\
mnist-s-advnone-ternweight & 784 & 10 & 3 & 4804 & 1 & \checkmark & 0.70 & 0.62 & \checkmark & 0.81 \\

\bottomrule[1.2pt]
\end{tabular}
}
\caption{%
  Detailed performance evaluation of global robustness verification for Binarized 
  Neural Networks (BNNs).
The table shows results for 13 models trained with CIFAR10 and MNIST image datasets.
(Note that the quantization step $qs$ is fixed to 0.61 for MNIST and 0.064 
for CIFAR10.)
%
} 
\label{tab:BNNs-pb}
\end{table}

%% file: conc.tex
\section{Conclusions \& Future Research} \label{sec:conc}

This paper presents simple arguments that demonstrate the shortcomings
of past work on deciding robustness, be it global or local. Similarly,
the paper uncovers critical shortcomings of certified robustness.
In addition, the paper also argues that possible attempts at solving
the identified shortcomings are not entirely satisfactory.
Given the results in this paper, we conclude that ongoing efforts for
delivering (local) robustness of ML models are misguided.

In contrast to the negative results presented in the paper, the paper
details recently proposed uses of robustness tools, building on the
connections between adversarial examples and explainability.
Furthermore, the negative results on robustness are used to shed light
on the properties of distance-restricted explanations of ML models.

Future work will further investigate the links between adversarial
examples and formal explanations. For example, key properties of
distance-restricted explanations, including worst-case size, will be
investigated, with the purpose of proposing distance-restricted
explanations as an alternative to probabilistic
explanations~\cite{kutyniok-jair21,barcelo-nips22,ihincms-ijar23}.

%% file: acks.tex
\section*{Acknowledgments}
  This work was supported by the AI Interdisciplinary Institute ANITI, 
  funded by the French program ``Investing for the Future -- PIA3''
  under Grant agreement no.\ ANR-19-PI3A-0004,  
  and
by the National Research Foundation, Prime Minister’s Office,
Singapore under its Campus for Research Excellence and Technological
Enterprise (CREATE) programme.

%% file: replbib.tex
\newtoggle{mkbbl}

%% file: togbbl.tex
\settoggle{mkbbl}{false}